\documentclass[11pt,a4paper]{article}
\usepackage[hyperref]{emnlp2020}
\usepackage{times}
\usepackage{latexsym}
\usepackage{subcaption}
\usepackage{caption}
\usepackage{graphicx}
\usepackage{multirow}
\usepackage{booktabs}
\usepackage{colortbl}
\usepackage{xcolor}
\usepackage{amsmath}
\usepackage{amssymb}
\usepackage{amsfonts}
\usepackage{algorithm}
\usepackage{algorithmic}
\usepackage{float}
\usepackage{bbm}
\usepackage{kotex}
\usepackage{mathtools}
\usepackage[utf8]{inputenc}
\usepackage{enumitem}
\usepackage{amsthm}
\newtheorem{theorem}{Theorem}

\usepackage{microtype}
\let\SUP\textsuperscript

\aclfinalcopy 
\def\aclpaperid{1022} 

\title{Context-Aware Answer Extraction in Question Answering}

\author{Yeon Seonwoo\SUP{$\dagger$}{\normalfont ,} Ji-Hoon Kim\SUP{$\ddagger \mathsection$}{\normalfont ,} Jung-Woo Ha\SUP{$\ddagger \mathsection$}{\normalfont ,} Alice Oh\SUP{$\dagger$}\\
  \SUP{$\dagger$}KAIST\\
  \SUP{$\ddagger$}NAVER AI LAB, \SUP{$\mathsection$}NAVER CLOVA\\
  {\tt yeon.seonwoo@kaist.ac.kr}\\
  {\tt\{genesis.kim,jungwoo.ha\}@navercorp.com}\\
  {\tt alice.oh@kaist.edu}\\
}

\date{}

\begin{document}
\maketitle
\begin{abstract}
Extractive QA models have shown very promising performance in predicting the correct answer to a question for a given passage.
However, they sometimes result in predicting the correct answer text but in a context irrelevant to the given question.
This discrepancy becomes especially important as the number of occurrences of the answer text in a passage increases.
To resolve this issue, we propose \textbf{BLANC} (\textbf{BL}ock \textbf{A}ttentio\textbf{N} for \textbf{C}ontext prediction) based on two main ideas: context prediction as an auxiliary task in multi-task learning manner, and a block attention method that learns the context prediction task.
With experiments on reading comprehension, we show that BLANC outperforms the state-of-the-art QA models, and the performance gap increases as the number of answer text occurrences increases.
We also conduct an experiment of training the models using SQuAD and predicting the supporting facts on HotpotQA and show that BLANC outperforms all baseline models in this zero-shot setting.
\end{abstract}

\section{Introduction}
\begin{figure}
    \centering
    \includegraphics[width=\linewidth]{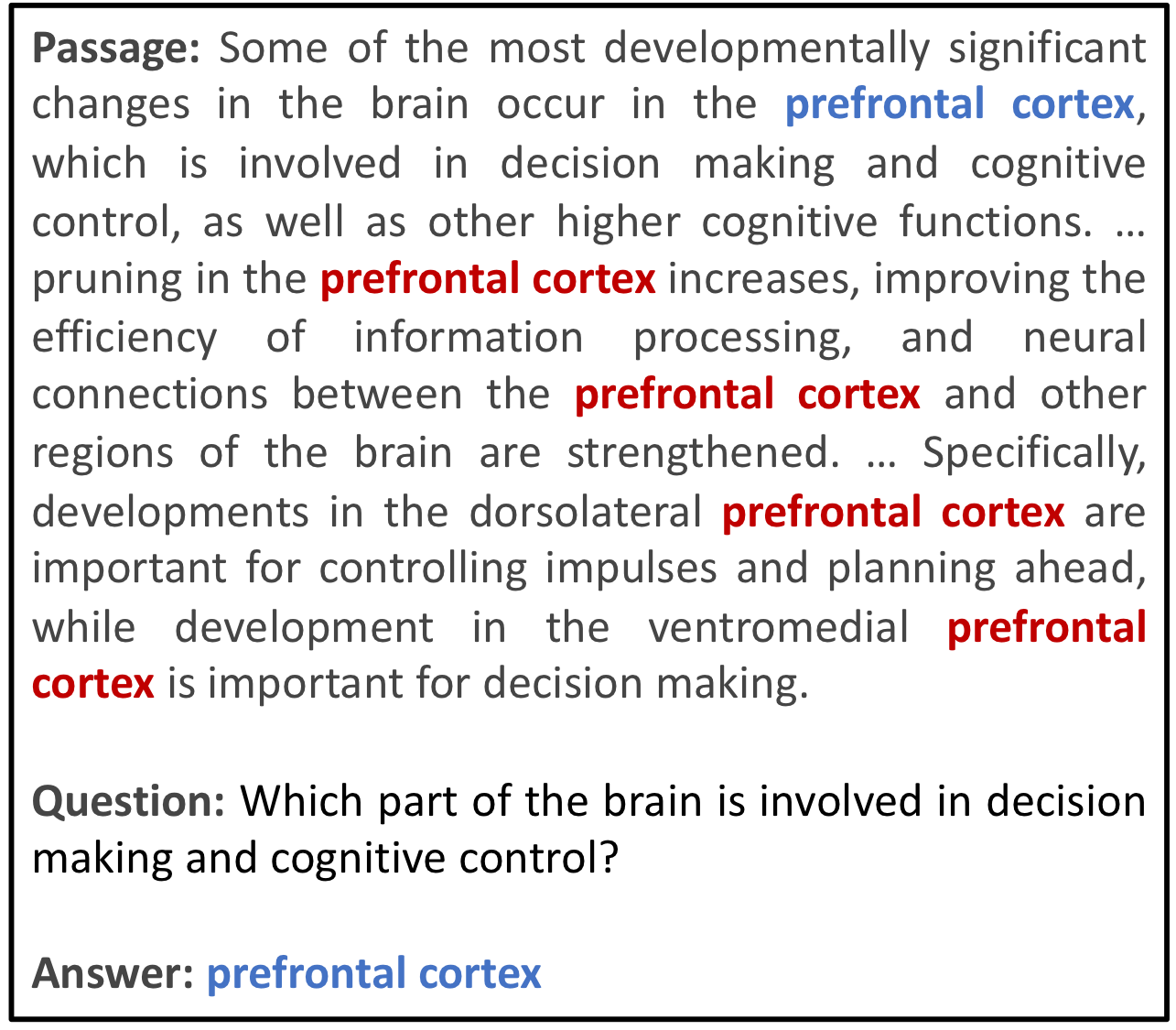}
    \caption{Example passage, question, and answer triple. This passage has multiple spans that are matched with the answer text. The first occurrence of ``prefrontal cortex" is the only answer-span within the context of the question.
    }
    \label{fig:example}
    \vspace{-1em}
\end{figure}

Question answering tasks require a high level of reading comprehension ability, which in turn requires a high level of general language understanding.
This is why the question answering (QA) tasks are often used to evaluate language models designed to be used in various language understanding tasks.
Recent advances in contextual language models brought on by attention \cite{hermann2015teaching, chen2016thorough, seobi, tay2018densely} and transformers \cite{vaswani2017attention} have led to significant improvements in QA, and these improvements show that better modeling of contextual meanings of words plays a key role in QA.

While these models are designed to select answer-spans in the relevant contexts from given passages, they sometimes result in predicting the correct answer text but in contexts that are irrelevant to the given questions.
Figure \ref{fig:example} shows an example passage where the correct answer text appears multiple times.
In this example, the only answer-span in the context relevant to the given question is the first occurrence of the ``prefrontal cortex" (in blue), and all remaining occurrences of the answer text (in red) show incorrect predictions.
Figure \ref{fig:sem-em} shows quantitatively, the discrepancy between predicting the correct answer text versus predicting the correct answer-span.
Using BERT \cite{devlin2019bert} trained on curated NaturalQuestions \cite{fisch2019mrqa}, we show the results of extractive QA task using exact match (EM) and Span-EM.
EM only looks for the text to match the ground truth answer, whereas Span-EM additionally requires the span to be the same as the ground truth answer-span.
Figure \ref{fig:sem-em} shows that BERT finds the correct answer text more than it finds the correct answer-spans, and this proportion of wrong predictions increases as the number of occurrences of answer text in a passage increases.

Tackling this problem is very important in more realistic datasets such as NaturalQuestions \cite{kwiatkowski2019natural}, where the majority of questions have more than one occurrence of the answer text in the passage.
This is in contrast with the SQuAD dataset, where most questions have a single occurrence of the answer. These details of the SQuAD \cite{rajpurkar2016squad}, NewsQA, and NaturalQuestions datasets \cite{fisch2019mrqa} are shown in Figure \ref{fig:n-answer-span}.

To address this issue, we define context prediction as an auxiliary task and propose a block attention method, which we call \textbf{BLANC} (\textbf{BL}ock \textbf{A}ttentio\textbf{N} for \textbf{C}ontext prediction) that explicitly forces the QA model to predict the context.
We design the context prediction task to predict soft-labels which are generated from given answer-spans.
The block attention method effectively calculates the probability of each word in a passage being included in the context with negligible extra parameters and inference time.
We provide the implementation of BLANC publicly available \footnote{\url{https://github.com/yeonsw/BLANC}}.

Adding context prediction and block attention enhances BLANC to correctly identify context related to a given question.
We conduct two types of experiments to verify the context differentiation performance of BLANC: extractive QA task, and zero-shot supporting facts prediction.
In the extractive QA task, we show that BLANC significantly increases the overall reading comprehension performance, and we verify the performance gain increases as the number of answer texts in a passage increases.
We verify BLANC's context-aware performance in terms of generalizability in the zero-shot supporting facts prediction task.
We train BLANC and baseline models on SQuAD1.1. and perform zero-shot supporting facts (supporting sentences in passages) prediction experiment on HotpotQA dataset \cite{yang2018hotpotqa}.
The results show that the context prediction performance that the model has learned from one dataset is generalizable to predicting the context of an answer to a question in another dataset.

Contributions in this paper are as follows: 
\begin{itemize}
    \item We show the importance of correctly identifying the answer-span to improving the model performance on extractive QA.
    \item We show that context prediction task plays a key role in the QA domain.
    \item We propose a new model BLANC that resolves the discrepancy between answer text prediction and answer-span prediction.
\end{itemize}

\begin{figure}[]
    \centering
    \includegraphics[width=0.85\linewidth]{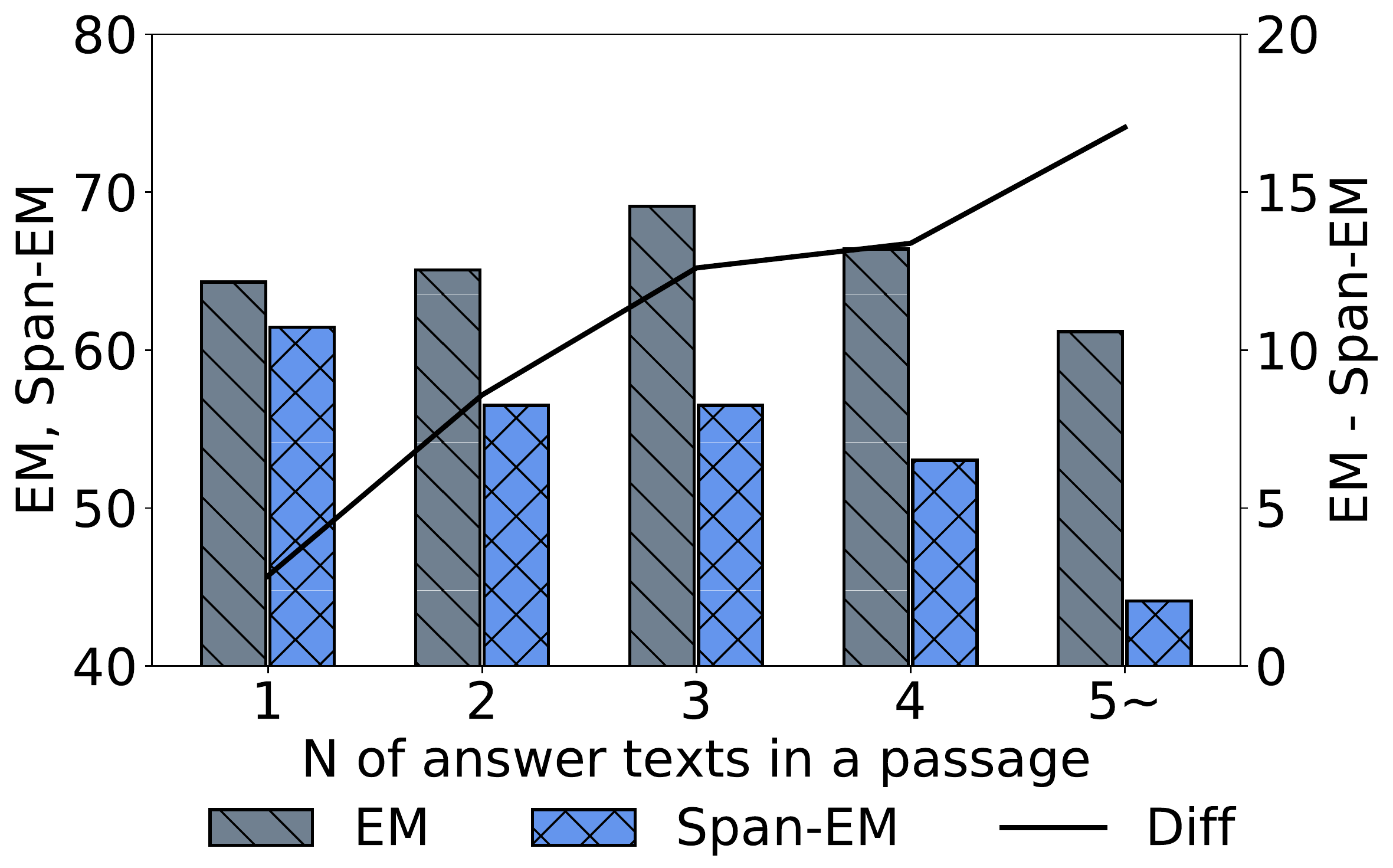}
    \caption{EM (text-matching) and Span-EM (span matching) of BERT on the groups divided by the number of answer text occurrences in a passage. Note: The difference for $N=1$ results from
    post-processing steps (removing articles) in EM evaluation.}
    \label{fig:sem-em}
    \vspace{-1em}
\end{figure}

\begin{figure}[t!]
    \centering
    \includegraphics[width=0.85\linewidth]{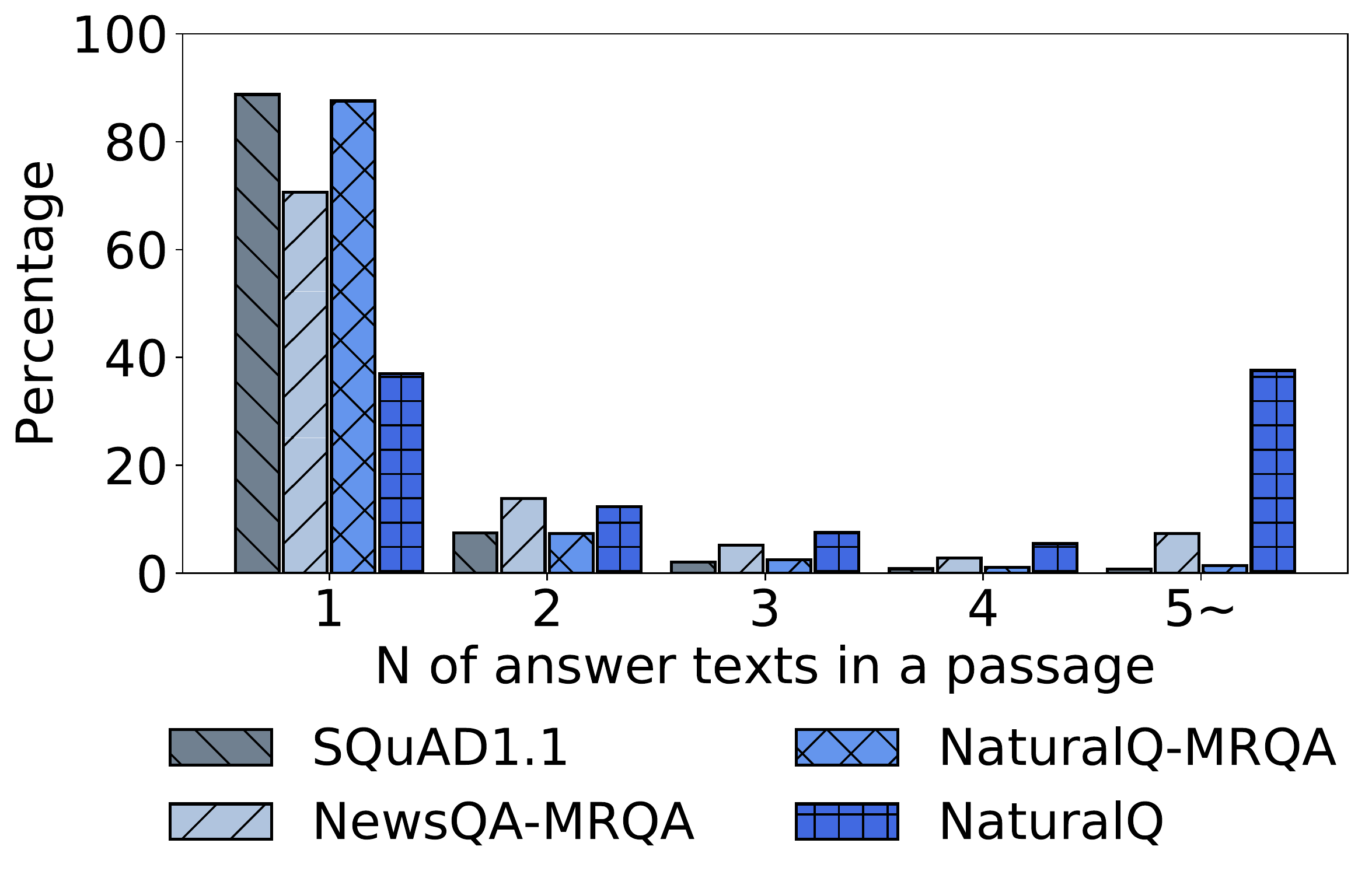}
    \caption{Proportions of questions with various numbers of the answer text in a passage. SQuAD has only a few examples for ($n\ge5$), while NaturalQuestions has a large proportion.}
    \label{fig:n-answer-span}
    \vspace{-1em}
\end{figure}
\section{Related Work}
Evidence in the form of documents, paragraphs, and sentences, has been shown to be necessary and effective in predicting the answers in open-domain QA \cite{chen2017reading, wang2018evidence, das2018multi, lee2019latent} and multi-hop QA \cite{yang2018hotpotqa, min2019multi, asai2020learning}.
One problem of identifying evidence in answering questions is the expensive cost in labeling the evidence.
Self-labeling with simple heuristics can be a solution to this problem, as shown in \newcite{choi2017coarse, li2018unified}.
Self-training is another solution, as presented in \newcite{niu2020self}.
In this paper, we propose self-generating soft-labeling method to indicate support words of answer texts, and train BLANC with the soft-labels.

Related but different from our work, \newcite{swayamdipta2018multi} and \newcite{min2019discrete} predict the answer-span when only the answer texts are provided and the ground truth answer-spans are not.
\newcite{swayamdipta2018multi} designs a model that benefits from aggregating information from multiple mentions of the answer text in predicting the final answer.
\newcite{min2019discrete} approach the problem of the lack of ground truth answer-spans with latent modeling of candidate spans.
Both of these papers tackle the problem of identifying the correct answer among multiple mentions of the answer text in datasets without annotations of the correct answer-spans.
Our work solves a different problem from the above-mentioned papers in that the golden answer-spans are provided.
\section{Model}
\begin{figure*}
    \centering
    \includegraphics[width=0.85\linewidth]{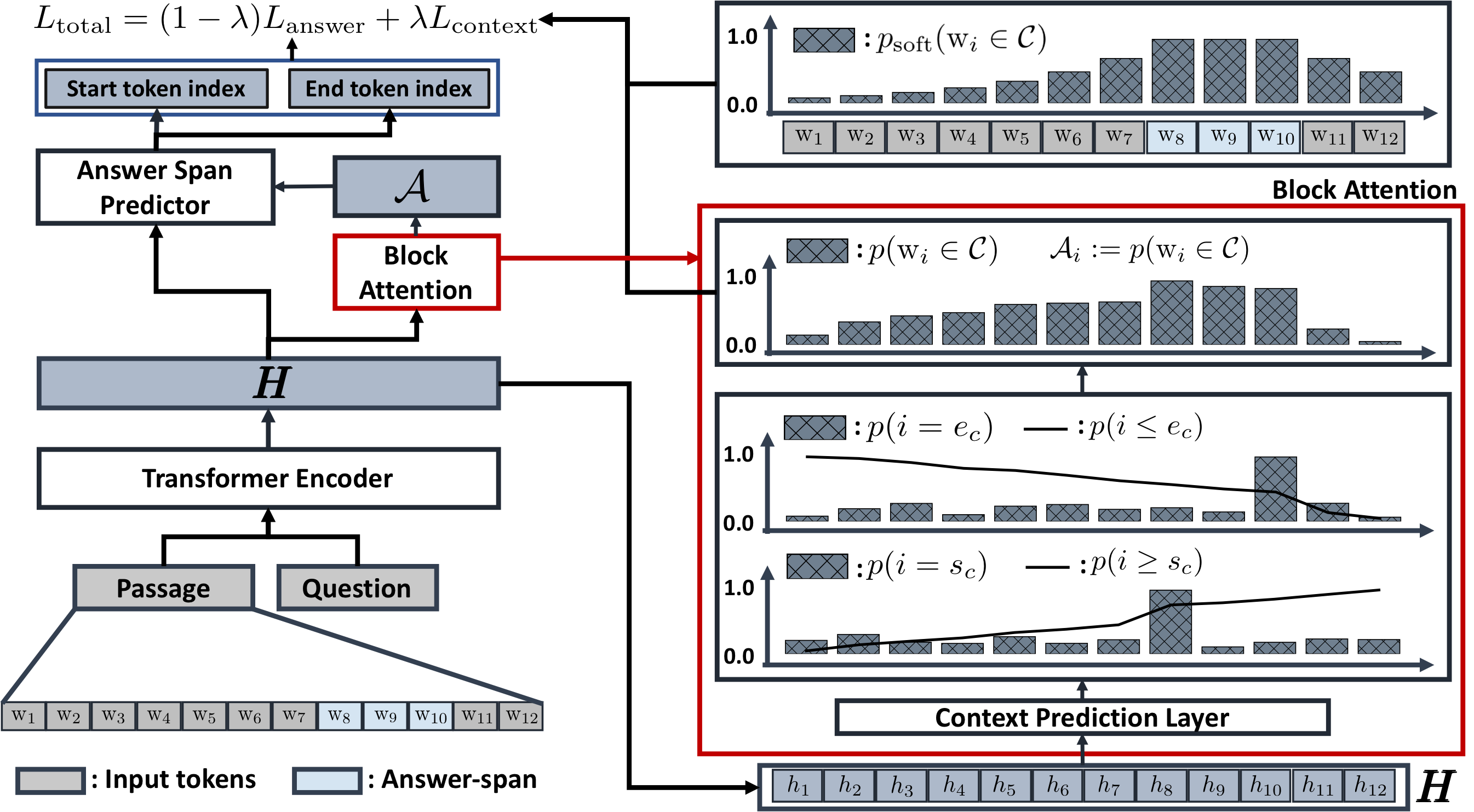}
    \caption{Schematic visualization of BLANC.
    Block attention model takes contextual vector representations from transformer encoder and predicts context words of an answer, $p(\text{w}_i \in \mathcal{C})$.
    We define loss function for context words with the prediction, 
    $p(\text{w}_i \in \mathcal{C})$ and the self-generated soft-label $p_{\text{soft}}(\text{w}_i \in \mathcal{C})$ defined in (\ref{eqn:pself}).
    Answer-span predictor takes $p(\text{w}_i \in \mathcal{C}$ and $\pmb{H}$ to predict an answer-span.
    We optimize our model in manner of multi-task learning of two tasks: answer-span prediction and context words prediction.
    }
    \label{fig:model}
    \vspace{-1em}
\end{figure*}
We propose BLANC based on two novel ideas: soft-labeling method for the context prediction and a block attention method that predicts the soft-labels.
Two important functionalities of BLANC are 1) calculating the probability that a word in a passage belongs to the context, which is in latent, and 2) enabling the probability to reflect spatial locality between adjacent words.
We provide an overall illustration of BLANC in Figure \ref{fig:model}.
\subsection{Notations}
In this section, we define the notations and the terms used in our study.
We denote a word at index $i$ in a passage with $\text{w}_i$.
We define the context of a given question as a segment of words in a passage and denote with $\mathcal{C}$.
In our circumstance, the context is latent.
We denote the start and end indices of a context with $s_{c}$ and $e_{c}$.
Training a block attention model to predict the context requires the labeling process for the latent context, and we define two probabilities for that, $p_{\text{soft}}(\text{w}_i \in \mathcal{C})$ and $p(\text{w}_i \in \mathcal{C})$.
$p_{\text{soft}}(\text{w}_i \in \mathcal{C})$ represents the self-generated soft-label that we assume as ground truth of the context, and $p(\text{w}_i \in \mathcal{C})$ is a block attention model's prediction.
We denote the start and end indices of a labeled answer-span with $s_a$ and $e_a$.
\subsection{Soft-labeling for latent context $\mathcal{C}$}\label{section:methodsl}
We assume words near an answer-span are likely to be included in the context of a given question.
From our assumption, we define the probability of words belong to the context, $p_{\text{soft}}(\text{w}_i \in \mathcal{C})$, which is used as a soft-label for the auxiliary context prediction task.
To achieve this, we hypothesize the words in an answer-span are included in the context and make the probability of adjacent words decrease with a specific ratio as the distance between answer-span and a word increases.
The soft-label for the latent context is as follows: 
\begin{equation}\label{eqn:pself}
    p_{\text{soft}}(\text{w}_i \in \mathcal{C}) =
    \begin{cases}
        1.0 & \text{if } i \in [s_a, e_a]\\
        q^{|i-s_a|} & \text{if } i < s_a\\
        q^{|i-e_a|} & \text{if } i > e_a,
    \end{cases}
\end{equation}
where $0 \le q \le 1$, and $q$ is a hyper-parameter for the decreasing ratio as the distance from a given answer-span.
For computational efficiency, we apply (\ref{eqn:pself}) to words bounded by certain window-size only, which is a hyper-parameter, on both sides of an answer-span.
This results in assigning $p_{\text{soft}}(\text{w}_i \in \mathcal{C})$ to 0 for the words outside the segment bounded by the window-size.

\subsection{Block Attention}\label{section:ba}
Block attention model calculates $p(\text{w}_i \in \mathcal{C})$ to predict the soft-label, $p_{\text{soft}}(\text{w}_i \in \mathcal{C})$, and localizes the correct index of an answer-span with $p(\text{w}_i \in \mathcal{C})$.
We embed spatial locality of $p(\text{w}_i \in \mathcal{C})$ to block attention model with the following steps: 1) predicting the start and end indices of context, $p(i = s_c)$ and $p(i = e_c)$, and 2) calculating $p(\text{w}_i \in \mathcal{C})$ with cumulative distribution of $p(i = s_c)$ and $p(i = e_c)$.
In the first step, at predicting the start and end indices, all encoder models that produce vector representation of words in a passage are compatible with the block attention model.
In this paper, we apply the same structure of the answer-span classification layer used in the transformer model \cite{devlin2019bert} to our context words prediction layer.
\begin{equation}
    \pmb{H} = \text{Encoder}(\text{Passage}, \text{Question})    
\end{equation}
Here, we denote $\pmb{H}$ as output vectors of transformer encoder and $\pmb{H}_j$ as output vector of $\text{w}_j$.
From $\pmb{H}$, we predict the start and end indices of the context:
\begin{equation}\label{eqn:context_pred}
    \begin{split}
        p(i=s_c) &= \frac{\text{exp}(\pmb{W_c}\pmb{H}_i + b_s^c)}{\sum_{j}{\text{exp}(\pmb{W_c}\pmb{H}_j + b_s^c)}},\\
        p(i=e_c) &= \frac{\text{exp}(\pmb{V_c}\pmb{H}_i + b_e^c)}{\sum_{j}{\text{exp}(\pmb{V_c}\pmb{H}_j + b_e^c)}},
    \end{split}
\end{equation}
where $\pmb{W_c}$, $\pmb{V_c}$, $b_s^c$, and $b_e^c$ represent weight and bias parameters for context prediction layer.
We calculate $p(\text{w}_i \in \mathcal{C})$ as multiplication of the probability of the word $\text{w}_i$ which appears after $s_c$ and that of the word $\text{w}_i$ which appears before $e_c$.
\begin{equation}\label{eqn:context}
    p(\text{w}_i \in \mathcal{C}) = p(i \ge s_c) \times p(i \le e_c).
\end{equation}
Here, we assume the independence between $s_c$ and $e_c$ for computational conciseness.
The cumulative distributions of $p(i \ge s_c)$ and $p(i \le e_c)$ are calculated with the following equations:
\begin{equation}
    \begin{split}
        p(i \ge s_c) &= \sum_{j \le i}{p(j = s_c)}\\
        p(i \le e_c) &= \sum_{j \ge i}{p(j = e_c)}.
    \end{split}
\end{equation}
We explicitly force the block attention model to learn context words of a given question by minimizing the cross-entropy of the two probabilities, $p(\text{w}_i \in \mathcal{C})$ and $p_{\text{soft}}(\text{w}_i \in \mathcal{C})$.
The loss function for the latent context is defined by the following equation:
\begin{equation}
    \begin{split}
        L_{\text{context}} = &-\sum_{1 \le i \le l}{p_{\text{soft}}(\text{w}_i \in \mathcal{C}) \log{p(\text{w}_i \in \mathcal{C})}}\\
        &-\sum_{1 \le i \le l}{p_{\text{soft}}(\text{w}_i \notin \mathcal{C}) \log{p(\text{w}_i \notin \mathcal{C})}},
    \end{split}
\end{equation}
\noindent where $l$ is the length of a passage.
By averaging $L_{\text{context}}$ across all train examples, we get the final context loss function.

\subsection{Answer-span Prediction}
BLANC predicts answer-span with the context probability, $p(\text{w}_i \in \mathcal{C})$.
We use the same answer-span prediction layer as BERT, but we multiply $p(\text{w}_i \in \mathcal{C})$ to the output of the encoder, $\pmb{H}$ to give attention at indices of answer-span within the context, $\mathcal{C}$.
\begin{equation}
    \begin{split}
        p(i=s_a) &= \frac{\text{exp}(\mathcal{A}_i\pmb{W_a}\pmb{H}_i + b_s^a)}{\sum_{j}{\text{exp}(\mathcal{A}_j\pmb{W_a}\pmb{H}_j + b_s^a)}},\\
        p(i=e_a) &= \frac{\text{exp}(\mathcal{A}_i\pmb{V_a}\pmb{H}_i + b_e^a)}{\sum_{j}{\text{exp}(\mathcal{A}_j\pmb{V_a}\pmb{H}_j + b_e^a)}},
    \end{split}
\end{equation}
where $\pmb{W_a}$, $\pmb{V_a}$, $b_s^a$, and $b_e^a$ represent weight and bias parameters for answer-span prediction layer, and $\mathcal{A}_i=p(\text{w}_i \in \text{context})$.
The loss function for answer-span prediction is defined by the following equation:
\begin{equation}
    \begin{split}
        L_{\text{answer}} = -\frac{1}{2}\{&\sum_{1 \le i \le l}{\mathbbm{1}(i=s_{a})\log{p(i=s_a)}}\\
        +&\sum_{1 \le i \le l}{\mathbbm{1}(i=e_{a})\log{p(i=e_a)}}\}.
    \end{split}
\end{equation}
$\mathbbm{1}(\text{condition})$ represents an indicator function that returns 1 if the condition is true and returns 0 otherwise.
By averaging $L_{\text{answer}}$ across all train examples, we get the final answer-span loss function.
We define our final loss function as the weighted sum of the two loss functions:
\begin{equation}
    L_{\text{total}} = (1-\lambda) L_{\text{answer}} + \lambda L_{\text{context}},
\end{equation}
\noindent where $\lambda$ is a hyper-parameter moderating the ratio of two loss functions.

\subsection{Property of Block Attention}
${p_{\text{soft}}(\text{w}_i \in \mathcal{C})}$ defined at (\ref{eqn:pself}) can be represented by the probability distributions calculated by block attention model, $p(\text{w}_i \in \mathcal{C})$.
We provide detailed proof in Appendix~\ref{section:appendix_softlabel}.
\section{Experimental Setup}\label{secion:exp_set}
We validate the efficacy of BLANC on two types of tasks: extractive QA and zero-shot supporting fact prediction. 
In the extractive QA, we evaluate the overall reading comprehension performance with three QA datasets, and we further analyze the ability of BLANC to discern relevant contexts on passages with multiple answer texts.
In zero-shot supporting facts prediction, we train QA models on SQuAD \cite{rajpurkar2016squad} and predict supporting facts (supporting sentences) of answers in HotpotQA \cite{yang2018hotpotqa}.
Due to our experimental computing resource limitation, we compare BLANC to baseline models trained in slightly modified hyperparameter settings instead of the results from their original papers.

\begin{table*}[t!]
\centering
\begin{tabular}{@{}llrcccc@{}}
\toprule
                            &                                    & \multicolumn{1}{l}{\#Param}  & Span-F1                                        & Span-EM                                        & F1                                             & EM                                             \\ \midrule
                            & BERT                               & 108M                         & 72.92 $\pm$ 0.36                                  & 60.63 $\pm$ 0.39                                  & 76.39 $\pm$ 0.26                                  & 64.48 $\pm$ 0.28                                  \\
                            & ALBERT                             & 17M                          & 72.66 $\pm$ 0.48                                  & 60.31 $\pm$ 0.49                                  & 75.89 $\pm$ 0.36                                  & 63.81 $\pm$ 0.37                                  \\
                            & RoBERTa                            & 124M                         & 75.07 $\pm$ 0.17                                  & 62.59 $\pm$ 0.14                                  & 78.54 $\pm$ 0.20                                  & 66.33 $\pm$ 0.09                                  \\
                            & SpanBERT                           & 108M                         & 75.16 $\pm$ 0.26                                  & 62.71 $\pm$ 0.37                                  & 78.31 $\pm$ 0.22                                  & 66.60 $\pm$ 0.31                                  \\
                            & \cellcolor[HTML]{C0C0C0}BLANC      & \cellcolor[HTML]{C0C0C0}108M & \cellcolor[HTML]{C0C0C0}\textbf{76.99 $\pm$ 0.09} & \cellcolor[HTML]{C0C0C0}\textbf{64.57 $\pm$ 0.12} & \cellcolor[HTML]{C0C0C0}\textbf{80.04 $\pm$ 0.06} & \cellcolor[HTML]{C0C0C0}\textbf{68.33 $\pm$ 0.09} \\ \cmidrule(l){2-7} 
                            & $\text{SpanBERT}_{\text{large}}$                      & 333M                         & 77.62 $\pm$ 0.10                                  & 65.28 $\pm$ 0.41                                  & 80.66 $\pm$ 0.11                                  & 69.14 $\pm$ 0.18                                  \\
\multirow{-7}{*}{NaturalQA} & \cellcolor[HTML]{C0C0C0}$\text{BLANC}_{\text{large}}$ & \cellcolor[HTML]{C0C0C0}333M & \cellcolor[HTML]{C0C0C0}\textbf{79.04 $\pm$ 0.27} & \cellcolor[HTML]{C0C0C0}\textbf{66.75 $\pm$ 0.14} & \cellcolor[HTML]{C0C0C0}\textbf{81.99 $\pm$ 0.16} & \cellcolor[HTML]{C0C0C0}\textbf{70.59 $\pm$ 0.12} \\ \midrule
                            & BERT                               & 108M                         & 83.36 $\pm$ 0.25                                  & 70.74 $\pm$ 0.43                                  & 88.10 $\pm$ 0.14                                  & 80.49 $\pm$ 0.28                                  \\
                            & ALBERT                             & 17M                          & 84.60 $\pm$ 0.13                                  & 72.04 $\pm$ 0.38                                  & 88.75 $\pm$ 0.20                                  & 81.05 $\pm$ 0.27                                  \\
                            & RoBERTa                            & 124M                         & 85.21 $\pm$ 0.25                                  & 72.82 $\pm$ 0.56                                  & 89.91 $\pm$ 0.16                                  & 82.53 $\pm$ 0.44                                  \\
                            & SpanBERT                           & 108M                         & 86.67 $\pm$ 0.16                                  & 74.08 $\pm$ 0.13                                  & 91.58 $\pm$ 0.09                                  & 84.97 $\pm$ 0.18                                  \\
                            & \cellcolor[HTML]{C0C0C0}BLANC      & \cellcolor[HTML]{C0C0C0}108M & \cellcolor[HTML]{C0C0C0}\textbf{86.89 $\pm$ 0.15} & \cellcolor[HTML]{C0C0C0}\textbf{74.69 $\pm$ 0.37} & \cellcolor[HTML]{C0C0C0}\textbf{91.87 $\pm$ 0.13} & \cellcolor[HTML]{C0C0C0}\textbf{85.30 $\pm$ 0.32} \\ \cmidrule(l){2-7} 
                            & $\text{SpanBERT}_{\text{large}}$                      & 333M                         & 88.27 $\pm$ 0.14                                  & 75.96 $\pm$ 0.22                                  & 93.22 $\pm$ 0.08                                  & 87.14 $\pm$ 0.11                                  \\
\multirow{-7}{*}{SQuAD1.1}  & \cellcolor[HTML]{C0C0C0}$\text{BLANC}_{\text{large}}$ & \cellcolor[HTML]{C0C0C0}333M & \cellcolor[HTML]{C0C0C0}\textbf{88.42 $\pm$ 0.17} & \cellcolor[HTML]{C0C0C0}\textbf{76.26 $\pm$ 0.31} & \cellcolor[HTML]{C0C0C0}\textbf{93.37 $\pm$ 0.05} & \cellcolor[HTML]{C0C0C0}\textbf{87.30 $\pm$ 0.10} \\ \midrule
                            & BERT                               & 108M                         & 59.18 $\pm$ 0.57                                  & 45.53 $\pm$ 0.55                                  & 65.07 $\pm$ 0.52                                  & 50.11 $\pm$ 0.50                                  \\
                            & ALBERT                             & 17M                          & 60.12 $\pm$ 0.36                                  & 46.54 $\pm$ 0.04                                  & 66.02 $\pm$ 0.35                                  & 51.18 $\pm$ 0.18                                  \\
                            & RoBERTa                            & 124M                         & 61.36 $\pm$ 0.63                                  & 47.43 $\pm$ 0.54                                  & 67.28 $\pm$ 0.63                                  & 52.36 $\pm$ 0.64                                  \\
                            & SpanBERT                           & 108M                         & 62.26 $\pm$ 0.22                                  & 48.04 $\pm$ 0.48                                  & 67.93 $\pm$ 0.26                                  & 52.85 $\pm$ 0.49                                  \\
                            & \cellcolor[HTML]{C0C0C0}BLANC      & \cellcolor[HTML]{C0C0C0}108M & \cellcolor[HTML]{C0C0C0}\textbf{64.39 $\pm$ 0.76} & \cellcolor[HTML]{C0C0C0}\textbf{50.60 $\pm$ 0.50} & \cellcolor[HTML]{C0C0C0}\textbf{70.31 $\pm$ 0.66} & \cellcolor[HTML]{C0C0C0}\textbf{55.52 $\pm$ 0.43} \\ \cmidrule(l){2-7} 
                            & $\text{SpanBERT}_{\text{large}}$                      & 333M                         & 63.43 $\pm$ 0.42                                  & 49.03 $\pm$ 0.13                                  & 69.06 $\pm$ 0.55                                  & 53.84 $\pm$ 0.27                                  \\
\multirow{-7}{*}{NewsQA}    & \cellcolor[HTML]{C0C0C0}$\text{BLANC}_{\text{large}}$ & \cellcolor[HTML]{C0C0C0}333M & \cellcolor[HTML]{C0C0C0}\textbf{66.48 $\pm$ 0.20} & \cellcolor[HTML]{C0C0C0}\textbf{52.39 $\pm$ 0.08} & \cellcolor[HTML]{C0C0C0}\textbf{72.36 $\pm$ 0.01} & \cellcolor[HTML]{C0C0C0}\textbf{57.40 $\pm$ 0.21} \\ \bottomrule
\end{tabular}
\caption{\label{tab:main} Reading comprehension performance of baseline models and BLANC.
We conduct experiments on three QA datasets: NaturalQ, SQuAD1.1, and NewsQA.
For all evaluation metrics, we report mean and standard deviation of three separate trials.
The results show that BLANC outperforms baseline models.
}
\vspace{-1em}
\end{table*}

\subsection{Datasets}
\paragraph{SQuAD:}
SQuAD1.1 \cite{rajpurkar2016squad} is a large reading comprehension dataset for QA.
Since the test set for SQuAD1.1 \cite{rajpurkar2016squad} is not publicly available, and their benchmark does not provide an evaluation on the span-based metric, we split train data (90\%/10\%) into new train/dev dataset and use development dataset as test dataset.
\paragraph{NewsQA \& NaturalQ:}
NewsQA \cite{trischler2017newsqa} consists of answer-spans to questions generated in a way that reflects realistic information seeking processes in the news domain.
NaturalQuestions \cite{kwiatkowski2019natural} is a QA benchmark in a real-world scenario with Google search queries for naturally-occurring questions and passages from Wikipedia for annotating answer-spans.
Due to computational limits, we use the curated versions of NewsQA and NaturalQ provided by \newcite{fisch2019mrqa}.
The curated datasets contain train and development set only, so we use the development set as the test set and build new train and dev sets from the train set (90\%/10\%).
\paragraph{HotpotQA:}
HotpotQA \cite{yang2018hotpotqa} aims to measure complex reasoning performance of QA models and requires finding relevant sentences from the given passages.
HotpotQA consists of passages, questions, answer, and corresponding supporting facts (sentences) for each answer.
We use the development set in HotpotQA.

\subsection{Evaluation Metrics}
F1 and EM are evaluation metrics widely used in existing QA models \cite{rajpurkar2016squad}.
These two metrics measure the number of overlapping tokens between the predicted answers and the ground truth answers.
Token matching evaluation treats as correct even answers in unrelated contexts, thus being insufficient to evaluate the context prediction performance.
As the alternatives, we propose span-EM and span-F1.
We modify the metric proposed in \newcite{kwiatkowski2019natural} to be suitable for our experiment setting.
\paragraph{Span-F1 and Span-EM:}
Span-F1 and span-EM are defined with overlapping indices between the predicted span and the ground truth span:
\begin{displaymath}
    \begin{split}
        \text{Span-P} &= {|[s_p, e_p] \cap [s_g, e_g]|}/{|[s_p, e_p]|}\\
        \text{Span-R} &= {|[s_p, e_p] \cap [s_g, e_g]|}/{|[s_g, e_g]|}\\
        \text{Span-F1} &= 2\times\frac{\text{Span-P}\times{\text{Span-R}}}{\text{Span-P}+{\text{Span-R}}}\\
        \text{Span-EM} &= \mathbbm{1}(s_p=s_g \wedge e_p=e_g)\\
    \end{split}
\end{displaymath}
Here, $s_p$ / $e_p$ represent the start/end indices of a predicted answer-span in a passage and $s_g$ / $e_g$ denote the start/end indices of the ground truth answer-span in a passage.
Span-EM measures exactly matched predicted spans, and Span-F1 quantifies the degree of overlap between the predicted answer-span and the ground truth span.

\subsection{Baselines}
\paragraph{BERT, RoBERTa, and ALBERT:}
BERT \cite{devlin2019bert}, RoBERTa \cite{liu2019roberta}, and ALBERT \cite{lan2019albert} are language models built upon the transformer encoder.
They use the same model structure, except for the bigger vocabulary size of RoBERTa.
Due to the computational limitation, we use 12-layer base models for BERT, and RoBERTa and 24-layer large model for ALBERT.
\paragraph{SpanBERT:}
SpanBERT \cite{joshi2020spanbert} has the same model structure and the same parameter size as BERT.
SpanBERT uses span-oriented pre-training for span representation.
Since the block attention is stacked on SpanBERT, and to provide detailed results of effectiveness of BLANC, we use both 12-layer SpanBERT-base and 24-layer SpanBERT-large.

\subsection{Hyper-parameter Settings}
We conduct experiments on limited hyper-parameter settings (e.g. max\_len, batch size), as we were limited by computational resources.  
We use the same hyperparameter settings across all baseline models and BLANC.
We set the training batch size to 8, learning-rate to $2\times e^{-5}$, the number of train epochs to 3, the max sequence length of transformer encoder to 384, warm-up proportion to 10\%, and we use the various optimizers used in the respective original papers.
We set $\lambda$ to 0.8, which is the optimal value as we show in Figure \ref{fig:abl_lmb}, for all experiments except the large model experiment on SQuAD1.1.
We set $\lambda=0.2$ in the large model experiment on SQuAD1.1.
We use different $q$, the decreasing ratio in (\ref{eqn:pself}), and different window-size for each dataset to reflect the average length of passages of each QA datasets.
We set $q=0.7$ and window-size to $2$ on SQuAD which contains relatively short passages, and $q=0.99$ and window-size to $3$ on the other two QA datasets where most passages are longer than SQuAD.
$q$ and window-size are optimized empirically.

\section{Results \& Discussion}\label{section:results}

\begin{figure}
    \centering
    \begin{subfigure}[b]{.40\textwidth}
        \centering
        \includegraphics[width=\linewidth]{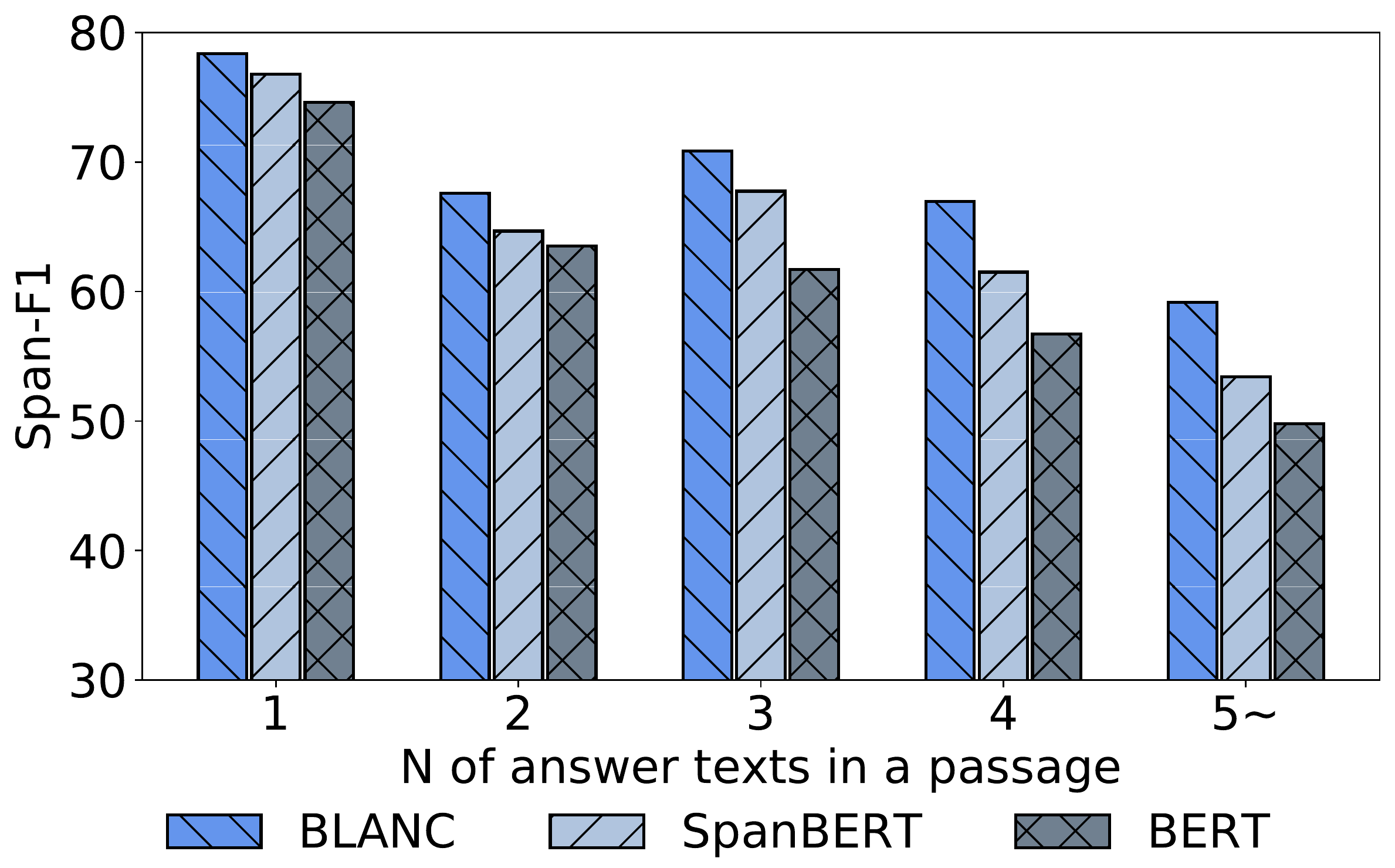}
    \end{subfigure}
    \begin{subfigure}[b]{.40\textwidth}
        \centering
        \includegraphics[width=\linewidth]{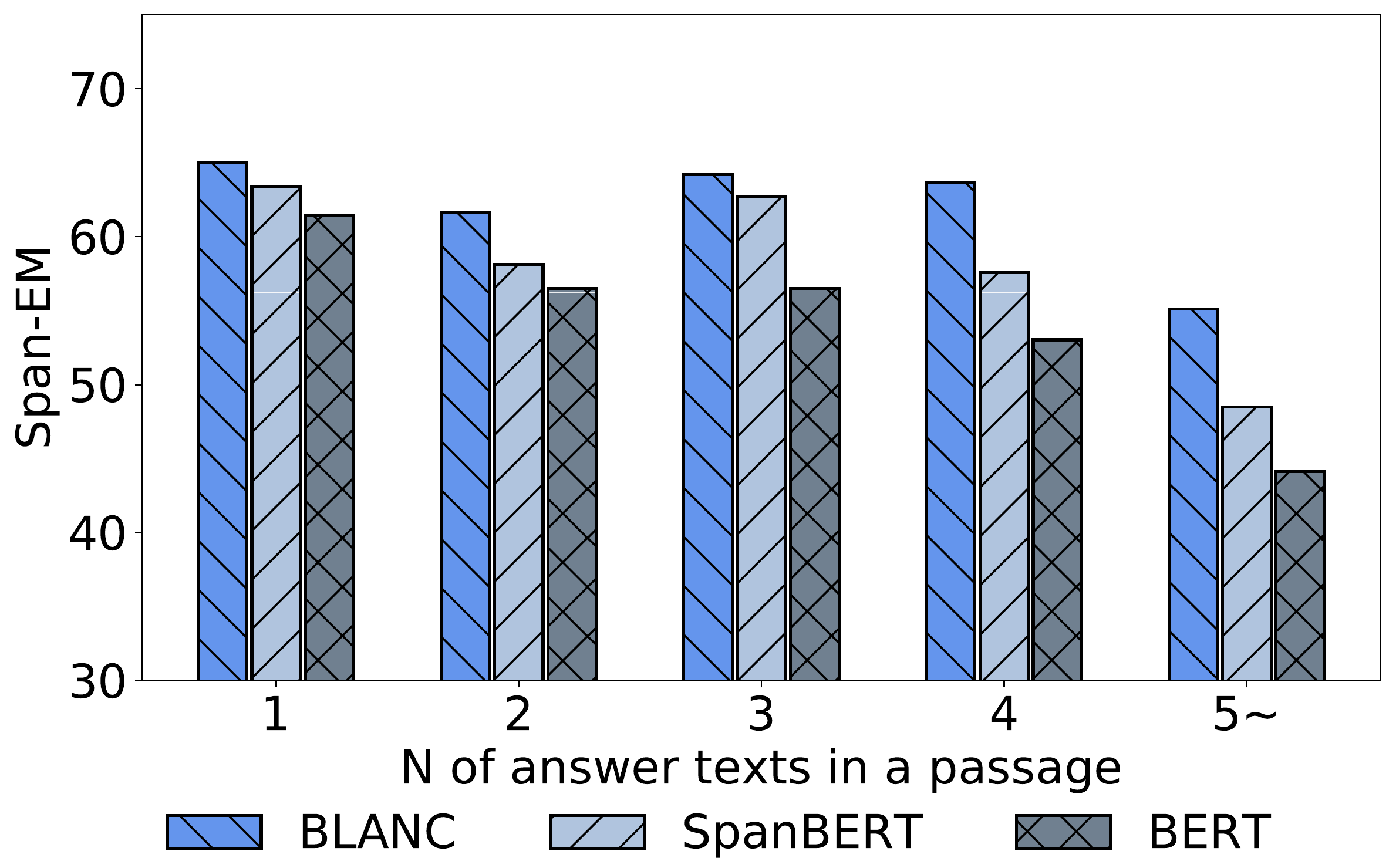}
    \end{subfigure}
    \caption{Span-F1 and Span-EM of baseline models and BLANC trained on NaturalQ.
    We categorize NaturalQ dataset into five groups by number answer texts appeared in a passage: $n=1, 2, 3, 4$, and $n\ge5$.
    BLANC outperforms baseline models on every groups and the performance gap increases as the number of answer texts in a passage increases.}
    \label{fig:perf-n-answer}
    \vspace{-1.4em}
\end{figure}

\begin{table}[]
\centering
\begin{tabular}{@{}lcc@{}}
\toprule
         & Span-F1               & Span-EM               \\ \midrule
RoBERTa  & 65.99 $\pm$ 0.92          & 60.12 $\pm$ 0.86          \\
SpanBERT & 63.47 $\pm$ 0.72          & 57.63 $\pm$ 0.79          \\
\rowcolor[HTML]{C0C0C0} 
BLANC      & \textbf{67.07 $\pm$ 0.36} & \textbf{61.43 $\pm$ 0.38} \\ \bottomrule
\end{tabular}
\caption{\label{tab:perf_on_n_ge_2}
Performance of BLANC on passages of NaturalQ that have answer texts two or more.}
\vspace{-1.5em}
\end{table}

We now present the results for the experiments described in the previous section. We describe the overall reading comprehension performance, highlighting the increased gain for passages with multiple mentions of the answer text. We show that BLANC outperforms other models for zero-shot supporting fact prediction. We also demonstrate the importance of the context prediction loss and the negligible extra parameter and inference time.  

\subsection{Reading Comprehension}\label{section:rc}
We verify the reading comprehension performance of BLANC with four evaluation metrics (F1, EM, Span-F1, and Span-EM) on three QA datasets: SQuAD, NaturalQ, and NewsQA.
We show the results in Table \ref{tab:main} which shows BLANC
consistently outperforms all comparison models including RoBERTa and SpanBERT.

We focus on the evaluation metric Span-EM which measures the exact match of the answer-span, and we further highlight the performance gain of BLANC over the most recent SpanBERT model, both base and large. 
On NaturalQ, BLANC outperforms SpanBERT by 1.86, whereas the performance difference between SpanBERT and RoBERTa is 0.12. On NewsQA, BLANC outperforms by 2.56, whereas the difference between SpanBERT and RoBRTa is 0.61. This pattern holds for the large models as well.

We now compare the performance gain between the datasets.
Recall that we showed in Figure \ref{fig:n-answer-span} the proportion of multi-mentioned answer is smallest in SQuAD, medium for NaturalQ-MRQA, and largest in NewsQA-MRQA. Reading comprehension results show the performance gap of BLANC and SpanBERT increases in the same order, verifying the effectiveness of BLANC on the realistic multi-mentioned datasets.

\subsection{Performance on Passages with Multi-mentioned Answers}
In Section \ref{section:rc}, we show Span-EM and EM of BLANC and baselines on the entire datasets.
However, the context discerning performance is only observed on passages with multiple mentions of the answer text.
We investigate the context-aware performance (distinguishing relevant context and irrelevant context) of BLANC by categorizing NaturalQ dataset by the number of occurrences of the answer text in a passage.
We subdivide the dataset into five groups: $n=1, 2, 3, 4$ and $n \ge 5$, where $n$ is the number of occurrences of the answer text in a passage.
Figure \ref{fig:perf-n-answer} presents Span-F1 and Span-EM on those subsets of the data.
BLANC outperforms SpanBERT and BERT across all subsets, and we show that the performance gain increases as $n$ increases.
In Table \ref{tab:perf_on_n_ge_2}, we explicitly show reading comprehension performance of BLANC on the question-answer pairs of passages with $n \ge 2$ from NaturalQ, and we confirm that block attention method increases context-aware performance of SpanBERT by 3.6 with Span-F1, and by 3.8 with Span-EM, which are larger improvements than the increments on the data including $n=1$ shown in Table \ref{tab:main}.

\begin{table}[]
\centering
\begin{tabular}{@{}lr@{}}
\toprule
\multicolumn{1}{c}{} & \multicolumn{1}{c}{Accuracy} \\ \midrule
BERT                 & 33.34 $\pm$ 0.82                 \\
$\text{ALBERT}_{\text{large}}$               & 35.62 $\pm$ 1.17                 \\
RoBERTa              & 37.93 $\pm$ 0.80                 \\
SpanBERT             & 34.79 $\pm$ 0.40                 \\
\rowcolor[HTML]{C0C0C0} 
BLANC                  & \textbf{39.80 $\pm$ 1.18}        \\ \bottomrule
\end{tabular}
\caption{\label{tab:hotpot}Performance on zero-shot supporting fact (supporting sentence) prediction by models trained with SQuAD1.1.
BLANC outperforms all other models.
}
\vspace{-1.5em}
\end{table}

\begin{figure}[]
    \centering
    \begin{subfigure}[b]{.23\textwidth}
        \includegraphics[width=\linewidth]{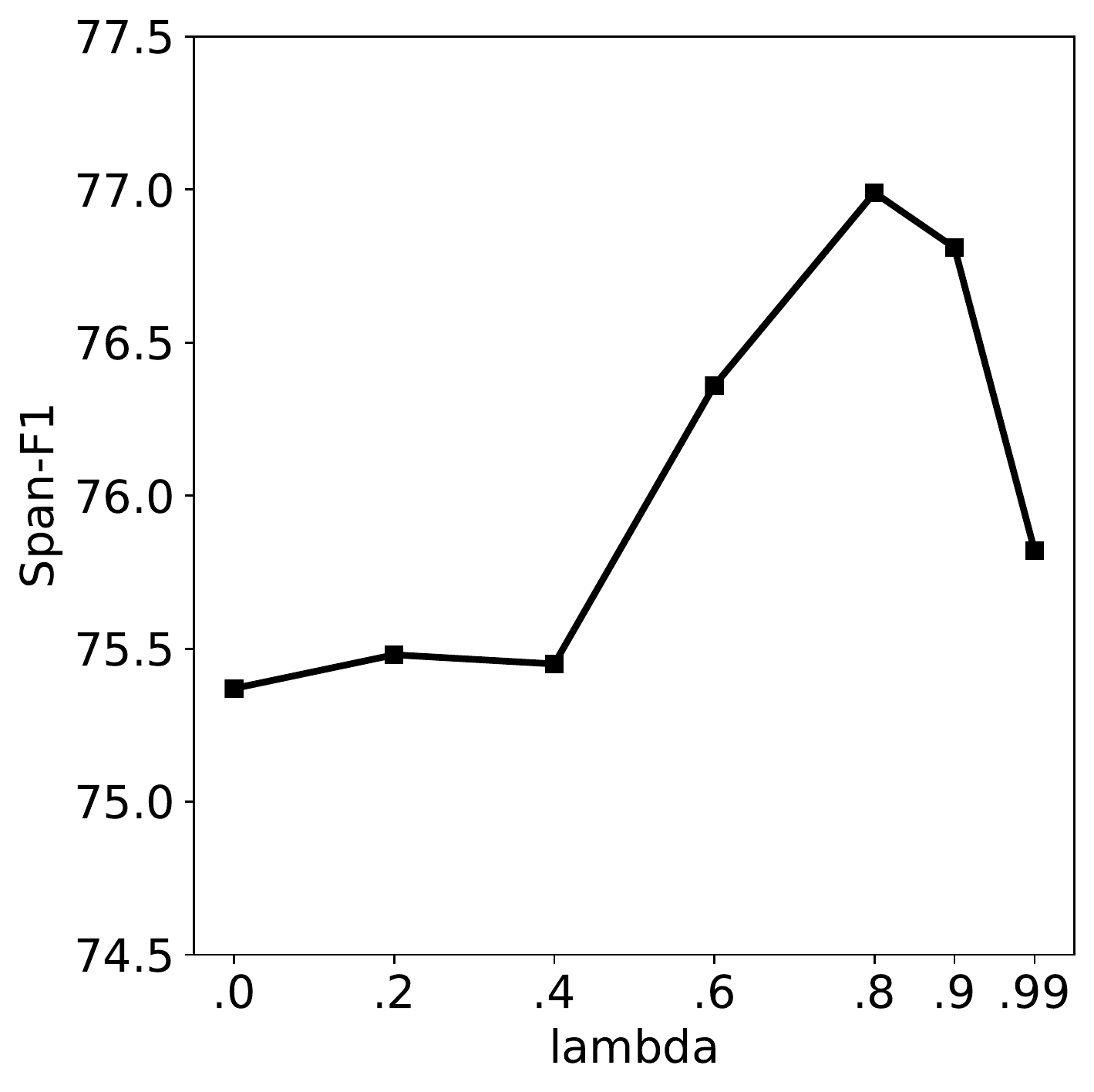}
        \label{fig:abl_span_f1}
        \vskip -0.2in
        \caption{Span-F1}
    \end{subfigure}
    \begin{subfigure}[b]{.23\textwidth}
        \includegraphics[width=\linewidth]{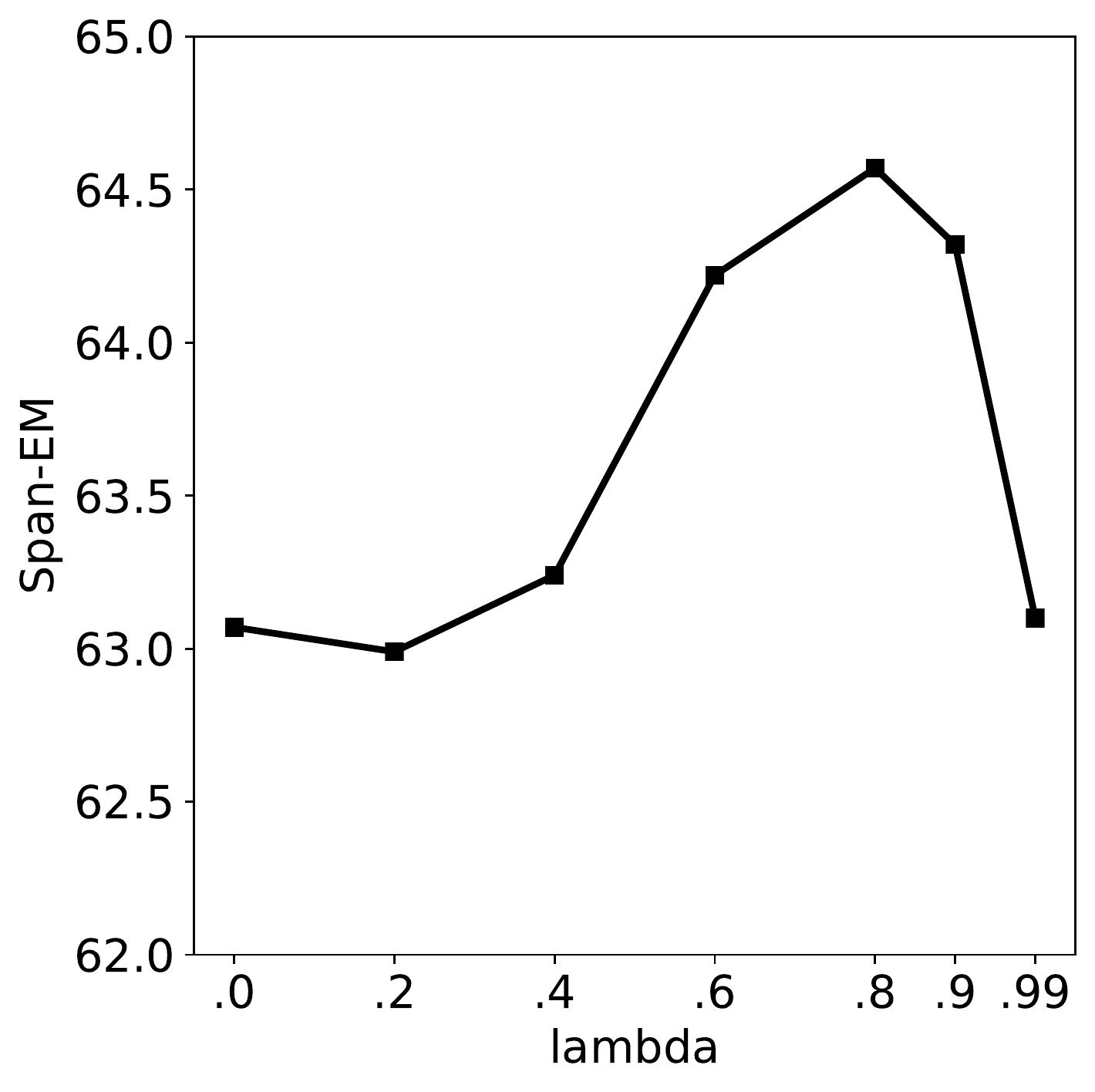}
        \label{fig:abl_span_em}\vskip -0.2in
        \caption{Span-EM}
    \end{subfigure}
    \caption{\label{fig:abl_lmb}
    Analysis on $\lambda$ for context word prediction for NaturalQ. We adjust $\lambda$, weight of ($L_{\text{context}}$), from $0.0$ to $0.99$ and report Span-F1 and Span-EM. Increasing $\lambda$ improves answer-span prediction until $\lambda=0.8$ and then decreases. This decrease is expected as the weight for ($L_{\text{answer}}$) becomes too small.}
    \vspace{-0.5em}
\end{figure}

\subsection{Supporting Facts Prediction}\label{section:supporting_facts_exp}
We present the results of the zero-shot supporting facts prediction task on HotpotQA dataset \cite{yang2018hotpotqa} in Table \ref{tab:hotpot}.
HotpotQA has ten passages and two supporting facts (sentences) for each question-answer pair.
Since HotpotQA has a different data format than the extractive QA datasets, we curate HotpotQA with the following steps.
We concatenate the ten passages to make one passage.
Two supporting facts exist in the passage.
By removing each one of them, we build two passages and each of passage contains one supporting fact.
We repeat this process for all examples in HotpotQA.
As a result, the curated dataset contains triples of one question, one supporting fact, and one passage.
We report the accuracy of models by checking if the supporting fact includes the predicted span.
We train baseline models and BLANC on SQuAD1.1 and test on the curated development set of HotpotQA dataset.
Table \ref{tab:hotpot} shows that BLANC captures sentence relevant to the given question better than other baseline models in zero-shot setting. This result shows that BLANC is capable of applying what it has learned from one dataset to predicting the context of an answer to a question in another dataset.

\begin{table}[]
\centering
\begin{tabular}{@{}lcc@{}}
\toprule
         & Train & Inf   \\ \midrule
BERT     & 1.00x & 1.00x \\
$\text{ALBERT}_{\text{large}}$   & 1.42x & 1.89x \\
RoBERTa  & 1.01x & 1.02x \\
SpanBERT & 1.00x & 1.00x \\
BLANC      & 1.04x & 1.00x \\ \bottomrule
\end{tabular}
\caption{\label{tab:complexity}Training and inference time of each model measured on the same number of QA pairs.}
\vspace{-1.5em}
\end{table}

\subsection{Analysis on $\lambda$}\label{section:abl}
We verify the relationship between reading comprehension performance and context word prediction task by conducting reading comprehension experiment with $\lambda=[0.2, 0.4, 0.6, 0.8, 0.9, 0.99]$.
The hyperparameter $\lambda$ represents weight of $L_{\text{context}}$ in the total loss function $L_{\text{total}}$.
Figure \ref{fig:abl_lmb} shows that the performance increases as $\lambda$ increases until it reaches $0.8$ and decreases after $\lambda=0.8$.
Leveraging the context word prediction task increases reading comprehension performance, and we show efficacy of BLANC.
As $\lambda$ increases, the weight on  $L_{\text{answer}}$ decreases, so we expect to see a decrease in performance as $\lambda$ becomes too large.

\subsection{Space and Time Complexity}
The additional parameters of block attention model come from Eq. (\ref{eqn:context_pred}) in Section \ref{section:ba}.
The number of parameters is $(768+1)*2=1538$ when the hidden dimension size of the transformer encoder is 768, and $1538$ is negligible considering the total number of parameters in BERT-base (108M).
The exact numbers of parameters of baseline models are presented in Table \ref{tab:main}.
Table \ref{tab:complexity} shows relative training and inference time of baseline models and BLANC.
We measure each model's train time on the same number of train steps and the inference time on the same number of passage-question pairs.
Since we use the 24-layer ALBERT-large model which has twice as many layers as other models, ALBERT requires the longest training/inference time,
despite its much smaller model size.
BLANC requires 4\% extra training time which includes the time to generate the soft-labels in (\ref{eqn:pself}) and the time to calculate the context word distribution in (\ref{eqn:context}).
For inference, BLANC requires negligible additional time on SpanBERT.
\section{Conclusion}
In this paper, we showed the importance of predicting an answer with the correct context of a given question.
We proposed BLANC with two novel ideas: context word prediction task and a block attention method that identifies an answer within the context of a given question.
The context words prediction task labels latent context words with the labeled answer-span and is used in a multi-task learning manner.
Block attention models the latent context words with negligible extra parameters and training/inference time.
We showed that BLANC increases reading comprehension performance, and we verify that the performance gain increases for complex examples (i.e., when the answer occurs two or more times in the passage).
Also, we showed the generalizability of BLANC and its context-aware performance with the zero-shot supporting fact prediction task on the HotpotQA dataset.
\section*{Acknowledgements}
This work was partly supported by NAVER Corp. and Institute for Information \& communications Technology Promotion(IITP) grant funded by the Korea government(MSIP) (2017-0-01780, The technology development for event recognition/relational reasoning and learning knowledge based system for video understanding).

\bibliography{emnlp2020}

\begin{thebibliography}{25}
\expandafter\ifx\csname natexlab\endcsname\relax\def\natexlab#1{#1}\fi

\bibitem[{Asai et~al.(2020)Asai, Hashimoto, Hajishirzi, Socher, and
  Xiong}]{asai2020learning}
Akari Asai, Kazuma Hashimoto, Hannaneh Hajishirzi, Richard Socher, and Caiming
  Xiong. 2020.
\newblock Learning to retrieve reasoning paths over {W}ikipedia graph for
  question answering.
\newblock In \emph{ICLR}.

\bibitem[{Chen et~al.(2016)Chen, Bolton, and Manning}]{chen2016thorough}
Danqi Chen, Jason Bolton, and Christopher~D Manning. 2016.
\newblock A thorough examination of the {CNN}/{D}aily mail reading
  comprehension task.
\newblock In \emph{ACL}.

\bibitem[{Chen et~al.(2017)Chen, Fisch, Weston, and Bordes}]{chen2017reading}
Danqi Chen, Adam Fisch, Jason Weston, and Antoine Bordes. 2017.
\newblock Reading {W}ikipedia to answer open-domain questions.
\newblock In \emph{ACL}.

\bibitem[{Choi et~al.(2017)Choi, Hewlett, Uszkoreit, Polosukhin, Lacoste, and
  Berant}]{choi2017coarse}
Eunsol Choi, Daniel Hewlett, Jakob Uszkoreit, Illia Polosukhin, Alexandre
  Lacoste, and Jonathan Berant. 2017.
\newblock Coarse-to-fine question answering for long documents.
\newblock In \emph{ACL}.

\bibitem[{Das et~al.(2018)Das, Dhuliawala, Zaheer, and McCallum}]{das2018multi}
Rajarshi Das, Shehzaad Dhuliawala, Manzil Zaheer, and Andrew McCallum. 2018.
\newblock Multi-step retriever-reader interaction for scalable open-domain
  question answering.
\newblock In \emph{ICLR}.

\bibitem[{Devlin et~al.(2019)Devlin, Chang, Lee, and
  Toutanova}]{devlin2019bert}
Jacob Devlin, Ming-Wei Chang, Kenton Lee, and Kristina Toutanova. 2019.
\newblock {BERT}: Pre-training of deep bidirectional transformers for language
  understanding.
\newblock In \emph{NAACL-HLT}.

\bibitem[{Fisch et~al.(2019)Fisch, Talmor, Jia, Seo, Choi, and
  Chen}]{fisch2019mrqa}
Adam Fisch, Alon Talmor, Robin Jia, Minjoon Seo, Eunsol Choi, and Danqi Chen.
  2019.
\newblock {MRQA} 2019 shared task: Evaluating generalization in reading
  comprehension.
\newblock In \emph{EMNLP 2019 MRQA Workshop}.

\bibitem[{Hermann et~al.(2015)Hermann, Kocisky, Grefenstette, Espeholt, Kay,
  Suleyman, and Blunsom}]{hermann2015teaching}
Karl~Moritz Hermann, Tomas Kocisky, Edward Grefenstette, Lasse Espeholt, Will
  Kay, Mustafa Suleyman, and Phil Blunsom. 2015.
\newblock Teaching machines to read and comprehend.
\newblock In \emph{NeurIPS}.

\bibitem[{Joshi et~al.(2020)Joshi, Chen, Liu, Weld, Zettlemoyer, and
  Levy}]{joshi2020spanbert}
Mandar Joshi, Danqi Chen, Yinhan Liu, Daniel~S Weld, Luke Zettlemoyer, and Omer
  Levy. 2020.
\newblock {S}pan{BERT}: Improving pre-training by representing and predicting
  spans.
\newblock \emph{Transactions of the Association for Computational Linguistics},
  8:64--77.

\bibitem[{Kwiatkowski et~al.(2019)Kwiatkowski, Palomaki, Redfield, Collins,
  Parikh, Alberti, Epstein, Polosukhin, Devlin, Lee
  et~al.}]{kwiatkowski2019natural}
Tom Kwiatkowski, Jennimaria Palomaki, Olivia Redfield, Michael Collins, Ankur
  Parikh, Chris Alberti, Danielle Epstein, Illia Polosukhin, Jacob Devlin,
  Kenton Lee, et~al. 2019.
\newblock {N}atural {Q}uestions: A benchmark for question answering research.
\newblock \emph{Transactions of the Association for Computational Linguistics},
  7:453--466.

\bibitem[{Lan et~al.(2019)Lan, Chen, Goodman, Gimpel, Sharma, and
  Soricut}]{lan2019albert}
Zhenzhong Lan, Mingda Chen, Sebastian Goodman, Kevin Gimpel, Piyush Sharma, and
  Radu Soricut. 2019.
\newblock {ALBERT}: A lite {BERT} for self-supervised learning of language
  representations.
\newblock In \emph{ICLR}.

\bibitem[{Lee et~al.(2019)Lee, Chang, and Toutanova}]{lee2019latent}
Kenton Lee, Ming-Wei Chang, and Kristina Toutanova. 2019.
\newblock Latent retrieval for weakly supervised open domain question
  answering.
\newblock In \emph{ACL}.

\bibitem[{Li et~al.(2018)Li, Li, and Wu}]{li2018unified}
Weikang Li, Wei Li, and Yunfang Wu. 2018.
\newblock A unified model for document-based question answering based on
  human-like reading strategy.
\newblock In \emph{AAAI}.

\bibitem[{Liu et~al.(2019)Liu, Ott, Goyal, Du, Joshi, Chen, Levy, Lewis,
  Zettlemoyer, and Stoyanov}]{liu2019roberta}
Yinhan Liu, Myle Ott, Naman Goyal, Jingfei Du, Mandar Joshi, Danqi Chen, Omer
  Levy, Mike Lewis, Luke Zettlemoyer, and Veselin Stoyanov. 2019.
\newblock {R}o{BERT}a: A robustly optimized bert pretraining approach.
\newblock \emph{arXiv}.

\bibitem[{Min et~al.(2019{\natexlab{a}})Min, Chen, Hajishirzi, and
  Zettlemoyer}]{min2019discrete}
Sewon Min, Danqi Chen, Hannaneh Hajishirzi, and Luke Zettlemoyer.
  2019{\natexlab{a}}.
\newblock A discrete hard {EM} approach for weakly supervised question
  answering.
\newblock In \emph{EMNLP-IJCNLP}.

\bibitem[{Min et~al.(2019{\natexlab{b}})Min, Zhong, Zettlemoyer, and
  Hajishirzi}]{min2019multi}
Sewon Min, Victor Zhong, Luke Zettlemoyer, and Hannaneh Hajishirzi.
  2019{\natexlab{b}}.
\newblock Multi-hop reading comprehension through question decomposition and
  rescoring.
\newblock In \emph{ACL}.

\bibitem[{Niu et~al.(2020)Niu, Jiao, Zhou, Yao, Xu, and Huang}]{niu2020self}
Yilin Niu, Fangkai Jiao, Mantong Zhou, Ting Yao, Jingfang Xu, and Minlie Huang.
  2020.
\newblock A self-training method for machine reading comprehension with soft
  evidence extraction.
\newblock \emph{arXiv}.

\bibitem[{Rajpurkar et~al.(2016)Rajpurkar, Zhang, Lopyrev, and
  Liang}]{rajpurkar2016squad}
Pranav Rajpurkar, Jian Zhang, Konstantin Lopyrev, and Percy Liang. 2016.
\newblock {SQuAD}: 100,000+ questions for machine comprehension of text.
\newblock In \emph{EMNLP}.

\bibitem[{Seo et~al.(2017)Seo, Kembhavi, Farhadi, and Hajishirzi}]{seobi}
Minjoon Seo, Aniruddha Kembhavi, Ali Farhadi, and Hananneh Hajishirzi. 2017.
\newblock Bi-directional attention flow for machine comprehension.
\newblock In \emph{ICLR}.

\bibitem[{Swayamdipta et~al.(2018)Swayamdipta, Parikh, and
  Kwiatkowski}]{swayamdipta2018multi}
Swabha Swayamdipta, Ankur~P Parikh, and Tom Kwiatkowski. 2018.
\newblock Multi-mention learning for reading comprehension with neural
  cascades.
\newblock In \emph{ICLR}.

\bibitem[{Tay et~al.(2018)Tay, Luu, Hui, and Su}]{tay2018densely}
Yi~Tay, Anh~Tuan Luu, Siu~Cheung Hui, and Jian Su. 2018.
\newblock Densely connected attention propagation for reading comprehension.
\newblock In \emph{NeurIPS}.

\bibitem[{Trischler et~al.(2017)Trischler, Wang, Yuan, Harris, Sordoni,
  Bachman, and Suleman}]{trischler2017newsqa}
Adam Trischler, Tong Wang, Xingdi Yuan, Justin Harris, Alessandro Sordoni,
  Philip Bachman, and Kaheer Suleman. 2017.
\newblock {NewsQA}: A machine comprehension dataset.
\newblock In \emph{Proceedings of the 2nd Workshop on Representation Learning
  for NLP}, pages 191--200.

\bibitem[{Vaswani et~al.(2017)Vaswani, Shazeer, Parmar, Uszkoreit, Jones,
  Gomez, Kaiser, and Polosukhin}]{vaswani2017attention}
Ashish Vaswani, Noam Shazeer, Niki Parmar, Jakob Uszkoreit, Llion Jones,
  Aidan~N Gomez, {\L}ukasz Kaiser, and Illia Polosukhin. 2017.
\newblock Attention is all you need.
\newblock In \emph{NeurIPS}.

\bibitem[{Wang et~al.(2018)Wang, Yu, Jiang, Zhang, Guo, Chang, Wang, Klinger,
  Tesauro, and Campbell}]{wang2018evidence}
Shuohang Wang, Mo~Yu, Jing Jiang, Wei Zhang, Xiaoxiao Guo, Shiyu Chang, Zhiguo
  Wang, Tim Klinger, Gerald Tesauro, and Murray Campbell. 2018.
\newblock Evidence aggregation for answer re-ranking in open-domain question
  answering.
\newblock In \emph{ICLR}.

\bibitem[{Yang et~al.(2018)Yang, Qi, Zhang, Bengio, Cohen, Salakhutdinov, and
  Manning}]{yang2018hotpotqa}
Zhilin Yang, Peng Qi, Saizheng Zhang, Yoshua Bengio, William Cohen, Ruslan
  Salakhutdinov, and Christopher~D Manning. 2018.
\newblock {H}otpot{QA}: A dataset for diverse, explainable multi-hop question
  answering.
\newblock In \emph{EMNLP}.

\end{thebibliography}
\bibliographystyle{acl_natbib}

\appendix

\section{Properties of Block Attention}
\subsection{Block Attention on a Soft-label}\label{section:appendix_softlabel}
\begin{theorem}
    There exist two probability distributions, $p(i = s_c)$ and $p(i = e_c)$, that makes $p(\text{w}_i \in \mathcal{C})$ equal to 
    ${p_{\text{soft}}(\text{w}_i \in \mathcal{C})}$, which is defined as follows:
    \begin{equation}
        p_{\text{soft}}(\text{w}_i \in \mathcal{C}) =
        \begin{cases}
            1.0 & \text{if } i \in [s_a, e_a]\\
            q^{|i-s_a|} & \text{if } s_w \le i < s_a\\
            q^{|i-e_a|} & \text{if } e_a < i \le e_w\\
            0.0 & \text{if } i < s_w \text{~~or~~ } i > e_w
        \end{cases}.
    \end{equation}
    Here, $q$ is the decreasing ratio, which satisfies $q \le 1.0$.
    $s_a$ and $e_a$ are the start and end indices of an answer-span.
    $s_a$ and $e_a$ satisfy $s_a \le e_a$. 
    $s_w$ and $e_w$ are the start and end indices of the segments bounded by certain window-size.
    $s_w$ and $e_w$ satisfy $s_w \le s_a$ and $e_a \le e_w$.
\end{theorem}
\begin{proof}
Based on the independent assumption between $s_c$ and $e_c$ in section \ref{section:ba}, $p(\text{w}_i \in \mathcal{C})$ becomes multiplication of two probability distributions as follows:
\begin{equation}
    p(\text{w}_i \in \mathcal{C}) = p(i \ge s_c) \times p(i \le e_c).
\end{equation}
Then, the following two cumulative distributions, $p(i \ge s_c)$ and $p(i \le e_c)$, make $p(\text{w}_i \in \mathcal{C})$ equal to $p_{\text{soft}}(\text{w}_i \in \mathcal{C})$:
\begin{equation}
    p(i \ge s_c) =
    \begin{cases}
        0.0 & \text{if } i < s_w\\
        p_{\text{soft}}(\text{w}_i \in \mathcal{C}) & \text{if } s_w \le i < s_a \\
        1.0 & \text{if } s_a \le i
    \end{cases},
\end{equation}
\begin{equation}
    p(i \le e_c) =
    \begin{cases}
        1.0 & \text{if } i \le e_a\\
        p_{\text{soft}}(\text{w}_i \in \mathcal{C}) & \text{if } e_a < i \le e_w \\
        0.0 & \text{if } e_w < i
    \end{cases}.
\end{equation}
Since block attention method can predict any form of $p(i=s_c)$ and $p(i=e_c)$, any soft-label can be represented by block attention method.
\end{proof}

\subsection{Block Attention on Multiple Spans}
Block attention model can be expanded to predict multiple spans.
\begin{theorem}
    Any form of the following ${p_{\text{multi-span}}(\text{w}_i \in \mathcal{C})}$, which has m-blocks, can be represented by the multiplication of a scaling factor, $k$, and the probability distribution calculated by block attention model, $p(\text{w}_i \in \mathcal{C})$.
    \begin{equation}
    p_{\text{multi-span}}(\text{w}_i \in \mathcal{C}) =
    \begin{cases}
        a & \text{if } i \in \mathcal{B}_1 \vee ... \vee i \in \mathcal{B}_m \\
        \epsilon & \text{otherwise}
    \end{cases}.
\end{equation}
Here, $\mathcal{B}_i$ is the set of indices of the $i$-th span, $\mathcal{B}_i = [s^b_i, e^b_i]$. $s^b_i$ and $e^b_i$ are the start and end indices of $\mathcal{B}_i$.
$\mathcal{B}_i$ satisfies $s^b_i \le e^b_i$ and $e^b_i < s^b_{i+1}$ for all $i$.
\end{theorem}
\begin{proof}
    Following two cumulative distributions and the scaling factor make $k \times p(i \ge s_c) \times p(i \le e_c)$ equal to $p_{\text{soft}}(\text{w}_i \in \mathcal{C})$ for all $i$.
    \begin{equation}
        p(i\ge s_c) = 
        \begin{cases}
            (\frac{\epsilon}{a})^{m} & \text{if } i < s^b_1\\
            (\frac{\epsilon}{a})^{m-j} & \text{if } s^b_j \le i < s^b_{j+1} \text{; } j \in [1, m)\\
            1.0 & \text{if } s^b_m \le i
        \end{cases}
    \end{equation}
    \begin{equation}
        p(i\le e_c) = 
        \begin{cases}
            1.0 & \text{if } i \le e^b_1\\
            (\frac{\epsilon}{a})^{j} & \text{if } e^b_j < i \le e^b_{j+1} \text{; } j \in [1, m)\\
            (\frac{\epsilon}{a})^{m} & \text{if } i > e^b_m\\
        \end{cases}
    \end{equation}
    \begin{equation}
        k = \epsilon \Big(\frac{a}{\epsilon}\Big)^{m}
    \end{equation}
    Since block attention model can predict any form of $p(i=s_c)$ and $p(i=e_c)$, $p_{\text{multi-span}}(w_i \in \mathcal{C})$ can be represented by the multiplication of a scaling factor and the probability distribution calculated by block attention model.
\end{proof}

\section{Semantic Similarity Between Context Words and Questions}
Soft-labeling method assumes that words near an answer-span are likely to be included in the context of a given question.
We provide the basis of this assumption with the question-word similarity experiment.
The question-word similarity is calculated with the cosine similarity between word vectors and question vectors.
We use word2vec vectors and calculate the question vectors by averaging word vectors in the questions.
Figure \ref{fig:qa_cos_sim} shows that words adjacent to the answer-spans have the most similar meaning to given questions.
Also, the similarity decreases as the distance between the words and the answer-spans increases.
From the results, we verify the assumption.
\begin{figure}[]
    \centering
    \includegraphics[width=0.8\linewidth]{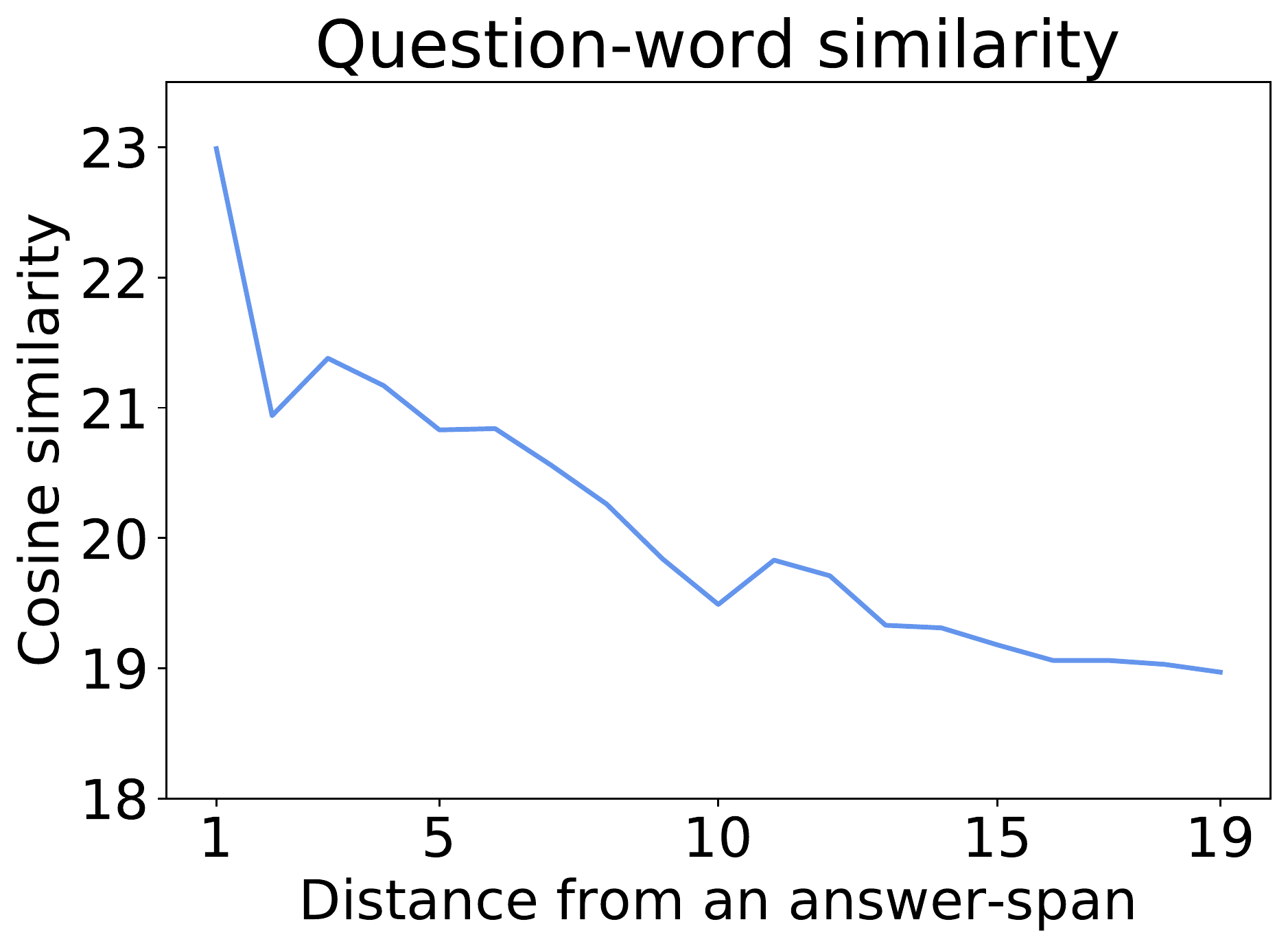}
    \label{fig:qa_cos_sim}
    \caption{\label{fig:qa_cos_sim} The semantic similarity between a given question and words in a passage. The x-axis represents the distance between a word and an answer-span. The y-axis represents the cosine similarity between the question and the word on 100 scale. Words near an answer-span are likely to have a similar meaning to a given question.}
\end{figure}

\section{Details about Hyperparameter Settings}
We vary window-size, and $\lambda$ to find the optimal hyperparameters of BLANC.
\subsection{Analysis on Window-size}
Table \ref{tab:ws} shows the performance of BLANC trained on NaturalQuestions with $\text{window-size} = [1, 2, 3, 4, 5, 7, 21]$.
AVG represents the average of the four performances.
BLANC shows the best AVG performance at $\text{WS}=3$, and we set window-size to 3 for NaturalQuestions and NewsQA experiments.

\begin{table}[]
\centering
\begin{tabular}{@{}cccccc@{}}
\toprule
WS                      & Span-F1 & Span-EM & F1    & EM    & AVG   \\ \midrule
\multicolumn{1}{c|}{1}  & 77.06   & 64.52   & 80.02 & 68.04 & 72.41 \\
\multicolumn{1}{c|}{2}  & 75.95   & 63.35   & 79.80 & 67.84 & 71.73 \\
\multicolumn{1}{c|}{3}  & 76.99   & 64.41   & 80.05 & 68.38 & 72.45 \\
\multicolumn{1}{c|}{4}  & 76.38   & 63.96   & 80.09 & 68.44 & 72.21 \\
\multicolumn{1}{c|}{5}  & 77.01   & 64.33   & 80.02 & 68.04 & 72.35 \\
\multicolumn{1}{c|}{7}  & 76.37   & 64.19   & 79.81 & 68.11 & 72.12 \\
\multicolumn{1}{c|}{21} & 76.65   & 64.14   & 79.96 & 68.06 & 72.20 \\ \bottomrule
\end{tabular}
\caption{\label{tab:ws} The performance of BLANC on NaturalQuestions. We vary window-size to find the optimal context size. AVG represents the average of the four performances.}
\end{table}

\subsection{Varying $\lambda$ on SQuAD1.1}
Table \ref{tab:lmb} shows the performance of BLANC with two different $\lambda$ settings on SQuAD1.1.
The results show that BLANC performs better at $\lambda=0.2$ than $\lambda=0.8$ (the optimal value for NaturalQuestions) on SQuAD1.1.
We set $\lambda$ to 0.2 in SQuAD1.1 experiments.

\begin{table}[]
\centering
\begin{tabular}{@{}ccc@{}}
\toprule
$\lambda$                   & Span-F1         & Span-EM         \\ \midrule
\multicolumn{1}{c|}{0.2} & 88.42 $\pm$ 0.17 & 76.26 $\pm$ 0.31 \\
\multicolumn{1}{c|}{0.8} & 88.30 $\pm$ 0.16 & 75.71 $\pm$ 0.30 \\ \bottomrule
\end{tabular}
\caption{\label{tab:lmb} The performance of BLANC on SQuAD1.1 with two different $\lambda$ settings.}
\end{table}

\end{document}